\documentclass{amsart} 

\usepackage{amsmath,amssymb}
\usepackage{amsthm}
\usepackage[alphabetic]{amsrefs}
\usepackage{indentfirst}
\usepackage{mathrsfs}
\usepackage{graphicx}
\usepackage{hyperref}

\usepackage{todonotes}

\theoremstyle{plain}
\newtheorem{theorem}{Theorem}
\newtheorem*{lemma}{Lemma}

\newtheorem{prop}[theorem]{Proposition}

\theoremstyle{definition}

\theoremstyle{remark}

\newcommand{\ud}{\,\mathrm{d}}
\newcommand{\rd}{\mathrm{d}}

\newcommand{\RR}{\mathbb{R}}
\renewcommand{\SS}{\mathbb{S}}

\newcommand{\bd}[1]{\boldsymbol{#1}}

\IfFileExists{mathabx.sty}%
  {\DeclareFontFamily{U}{mathx}{\hyphenchar\font45}%
   \DeclareFontShape{U}{mathx}{m}{n}{<->mathx10}{}%
   \DeclareSymbolFont{mathx}{U}{mathx}{m}{n}%
   \DeclareFontSubstitution{U}{mathx}{m}{n}%
   \DeclareMathAccent{\widebar}{0}{mathx}{"73}%
}{%
  \PackageWarning{mathabx}{%
    Package mathabx not available, therefore\MessageBreak substituting
    widebar with overline\MessageBreak }%
  \newcommand{\widebar}[1]{\overline{#1}}%
}

\newcommand{\mc}[1]{\mathcal{#1}}

\newcommand{\abs}[1]{\lvert#1\rvert}

\newcommand{\average}[1]{\langle#1\rangle}

\title{Neural collapse with cross-entropy loss}
\author[]{Jianfeng Lu}
\address[JL]{Department of Mathematics, Department of Physics, and Department of Chemistry,
Duke University, Box 90320, Durham NC 27708, USA}
\email{jianfeng@math.duke.edu}

\author[]{Stefan Steinerberger}
\address[SS]{Department of Mathematics, University of Washington, Seattle, WA 98195, USA} \email{steinerb@uw.edu}

 \thanks{The work of
  J.L.~is partially supported by the National Science Foundation via
  grants DMS-2012286 and CCF-1934964. S.S.~is partially supported by
  the NSF (DMS-1763179) and the Alfred P. Sloan Foundation. J.L.~would
  also like to acknowledge helpful discussions with Joan Bruna and
  L\'ena{\"\i}c Chizat, and would also thank the Flatiron Institute
  Collaboration of Mathematics of Deep Learning, from which he learned
  the neural collapse behavior.}

\begin{document}

\begin{abstract}
  We consider the variational problem of cross-entropy loss with $n$
  feature vectors on a unit hypersphere in $\RR^d$. We prove that when
  $d \geq n - 1$, the global minimum is given by the simplex
  equiangular tight frame, which justifies the neural collapse
  behavior. We also prove that as $n \to \infty$ with fixed $d$, the
  minimizing points will distribute uniformly on the hypersphere and
  show a connection with the frame potential of Benedetto \&
  Fickus.
\end{abstract}
\maketitle

\section{Introduction and Results}
\subsection{Introduction}
We consider the following variational problem 
\begin{equation}\label{eq:symvarp}
  \min_{u} \mc{L}_\alpha(u) := \min_{u}\; \sum_{i=1}^n \log \Bigl( 1 + \sum_{j=1 \atop j\neq i }^n e^{\alpha (\average{u_j, u_i} - 1)} \Bigr) . 
\end{equation}
where $\alpha >0 $ is a parameter and for $i = 1, \ldots, n$,
$u_i \in \RR^d$ such that $\|u_i\| = 1$.  Here and in the sequel, we
use $\|\cdot\|$ to denote the Euclidean norm of a vector, and
$\average{\cdot, \cdot}$ denotes the Euclidean inner product. The
question we would like to address in this note is the solution
structure of such variational problems. The problem has several
motivations from some recent works in the literature of machine
learning.\\

Our main motivation comes from the very nice recent paper
\cite{papyan2020prevalence}. In that paper, the authors proposed and
studied the neural collapse behavior of training of deep neural
networks for classification problems. Following the work
\cite{papyan2020prevalence} by choosing a cross-entropy loss, while
taking unconstrained features (i.e., not parametrized by some
nonlinear functions like neural networks) to be vectors on the unit sphere in $\RR^d$, this
amounts to the study of the variational problem
\begin{equation}\label{eq:varp}
  \min_{u, v} \mc{L}(u, v) := \min_{u, v}\; \sum_{i=1}^n \log \Biggl( \frac{\sum_{j=1}^n e^{\average{v_j, u_i}}}{e^{\average{v_i, u_i}}}\Biggr). 
\end{equation}
where $u_i, v_i \in \RR^d$ such that $\|u_i\| = 1$ for each $i$.  Note
that the model in \cite{papyan2020prevalence} also contains a bias
vector, so that $\average{v_j, u_i}$ in \eqref{eq:varp} is replaced by
$\average{v_j, u_i} + b_j$, for $b \in \RR^n$. We drop the bias to
remove some degeneracy of the problem for simplicity. Another, more
crucial, difference is that in actual deep learning, as considered in~\cite{papyan2020prevalence}, the feature vectors $u_i$ are given by
output of deep neural networks acting on the input data, this would
make the variational problem much harder to analyze and thus we will
only study the simplified scenario.

\smallskip 

\subsection{Equiangular tight frame as minimizer.}
The connection between the two variational problems is evident, as
\eqref{eq:symvarp} can be viewed as a symmetric version of
\eqref{eq:varp}: In particular, if we choose $v_i = \alpha u_i$ for
some parameter $\alpha > 0$, then
$\mc{L}(\alpha u, u) = \mc{L}_{\alpha}(u)$. In fact, we will prove
that the minimum of \eqref{eq:varp} is indeed achieved by such
symmetric solutions. There is a small caveat though as one can take
the norm of $v$ to infinity (or $\alpha \to \infty$ for the symmetric
problem) to reduce the loss. Thus, in order to characterize better the
solution structure, we will consider the problem for a fixed scaling
of $v$, and in fact $\|v_i\| = 1$ (the case of $\|v_i\| = \alpha$
will be discussed below). We show that the solution of
the variational problem is given by a simplex equiangular tight frame
(ETF). This proves the neural collapse behavior for
\eqref{eq:varp}, which provides some justification to the observation
of such behavior in deep learning.

\begin{theorem}\label{thm:crossentropy}
  Consider the variational problem
  \begin{equation*}
    \begin{aligned}
      & \min_{u, v} \mc{L}(u, v) \\
      & \text{such that } u_i, v_i \in \RR^d, \, \|u_i\| = \|v_i\|
      = 1, \quad i =1, \cdots, n. 
    \end{aligned}
  \end{equation*}
  If $d \geq n-1$, the global minimum of the problem corresponds to
  the case where $\{u_i\}_{i=1}^{n}$ form a simplex equiangular tight frame and $u_i = v_i$ for all
  $i = 1, \cdots, n$.
\end{theorem}

We remark that similar results have been proved for different loss
functions: for a large deviation type loss function in
\cite{papyan2020prevalence} and for a $L^2$-loss function in
\cite{mixon2020neural}, both for models with unconstrained feature
vectors (i.e., without neural network parametrization of $u_i$'s).
After making the first version of our paper available, we were informed that an
equivalent result has been obtained in \cite{ewojtowytsch}, based on a
quite different proof.

\smallskip

In Theorem~\ref{thm:crossentropy}, the restriction of the scale of
$\|v_i\| = 1$ does not in fact sacrifice generality. As we will
comment towards the end of the proof, if instead $\|v_i\| \leq \alpha$
is assumed, the solution would be given by $v_i = \alpha u_i$. This is
related to the following result for the symmetric problem
\eqref{eq:symvarp}.

\begin{theorem}\label{thm:symvarp}
  Consider the variational problem
  \begin{equation*}
    \begin{aligned}
      & \min_{u} \mc{L}_{\alpha}(u) \\
      & \text{such that } u_i \in \RR^d, \, \|u_i\| 
      = 1, \quad i =1, \cdots, n. 
    \end{aligned}
  \end{equation*}
  If $d \geq n-1$, then for any $\alpha > 0$, the global minimum of
  the problem corresponds to the case where $\{u_i\}_{i=1}^{n}$ form a simplex equiangular tight frame.
\end{theorem}

We prove Theorem~\ref{thm:symvarp} first in
Section~\ref{sec:symmetric}; the idea of the proof extends to that of
Theorem~\ref{thm:crossentropy}, which will be presented in
Section~\ref{sec:crossentropy}.

\subsection{More vectors than dimensions.} In the above results, the assumption that $p \geq n-1$ (or
equivalently $n \leq p + 1$) is crucial as only then it is possible to
place $n$ vectors on the unit sphere in $\mathbb{\RR}^p$ such that
these vectors form a simplex equiangular tight frame. It is natural to ask what happens
when $n \geq p + 2$. In this case, \eqref{eq:symvarp} is related to
loss functions used in unsupervised learning and self-supervised
learning, such as those used in Siamese networks
\cite{chopra2005learning} and word2vec \cite{ariel,mikolov2013efficient}. In
particular, the spherical contrastive loss considered in
\cite{wang2020understanding}, for which the goal is to embed many
points on a hypersphere such that the points are ``uniformly
distributed'', coincides with \eqref{eq:symvarp} when unconstrained
feature vectors are used.\\

For general $n \geq p + 2$, the study of the structure of the
minimizer seems difficult, but in the asymptotic regime
$n \to \infty$, we have the following theorem states that indeed the
points will uniformly distributed on the sphere.

\begin{theorem}\label{thm:uniform}
  Consider the variational problem
  \begin{equation*}
    \begin{aligned}
      & \min_{u} \mc{L}_{\alpha}(u) \\
      & \text{such that } u_i \in \RR^d, \, \|u_i\| = 1, \quad i =1,
      \cdots, n.
    \end{aligned}
  \end{equation*}
  Let $\mu_n$ be the probability measure on $\SS^d$ generated by a
  minimizer
  \begin{equation*}
    \mu_n = \frac{1}{n} \sum_{i=1}^n \delta_{u_i}, 
  \end{equation*}
  then for any $\alpha > 0$, as $n \to \infty$, $\mu_n$ 
  converges weakly to the uniform measure on $\SS^d$.
\end{theorem}

The proof of theorem will be presented in
Section~\ref{sec:uniform}. It uses the following proposition which
characterizes the minimizer for a relaxed version of the variational
problem defined for probability measures on $\SS^d$. With some abuse
of notation, for $\mu \in \mathcal{P}(\SS^d)$, we denote
\begin{equation}
  \mc{L}_{\alpha}(\mu) := \int_{\SS^d} \log \left( \int_{\SS^d}
    e^{\alpha ( \average{x, y} - 1)}
    \mu(\rd y)\right) \mu(\rd x). 
\end{equation}
It is easy to check that $\mc{L}_{\alpha}$ acting on
$\mu_n = \frac{1}{n} \sum_{i=1}^n \delta_{u_i}$ is equivalent (up to
some additive constants) to the objective function defined in
\eqref{eq:symvarp} evaluated at the point configuration $\{u_i\}$. The
following proposition states that the unique minimizer of
$\mc{L}_{\alpha}$ on $\mc{P}(\SS^d)$ is given by uniform probability
measure; this fact was established in \cite[Theorem
1]{wang2020understanding} using a rather different approach. Our
proof, deferred to Section~\ref{sec:uniform}, which is based on
variational arguments, seems simpler in comparison.

\begin{prop}\label{prop:uniformsphere}
  The unique minimizer of the variational problem
  \begin{equation*}
    \inf_{\mu \in \mathcal{P}(\SS^d)} \mc{L}_{\alpha}(\mu) 
  \end{equation*}
  is the uniform probability measure on $\SS^d$.
\end{prop}

One natural question is how quickly sets of $n$ points can have an energy that is comparable to the energy of the uniform measure. We will show that $n$ points can have an energy that is super-exponentially close to the energy of the flat distribution. Note first that
\begin{align*}
  \mc{L}_{\alpha}(\mu) &= \int_{\SS^d} \log \left( \int_{\SS^d} e^{\alpha ( \average{x, y} - 1)} \mu(\rd y)\right) \mu(\rd x) \\
  &=  \int_{\SS^d} \log \left( e^{-\alpha} \int_{\SS^d} e^{\alpha  \average{x, y} } \mu(\rd y)\right) \mu(\rd x) \\
  &= - \alpha +   \int_{\SS^d} \log \left( \int_{\SS^d} e^{\alpha  \average{x, y} } \mu(\rd y)\right) \mu(\rd x)
\end{align*}
and that we can therefore study this equivalent but slightly more symmetric functional. In particular, when studying the case where $\mu$ is given by the sum of Dirac measure in a finite number of points, we can obtain an even more symmetric upper bound by applying Jensen's inequality 
$$  \sum_{i =1}^{n} \log \left( \sum_{j=1}^{n} e^{\alpha \left\langle x_i, x_j \right\rangle} \right) \leq n \log\left(   \frac{1}{n} \sum_{i, j=1}^{n} e^{\alpha \left\langle x_i, x_j \right\rangle} \right).$$
As it turns out, the particular structure of this upper bound allows us to obtain very precise bounds for the minimal energy of optimal configurations.

\begin{theorem} \label{upper} Let $\alpha > 0$ be fixed. For some $c_{\alpha, d} > 0$ and all $n$ sufficiently large, there exist sets of points $\bigl\{ x_1, \dots, x_n \bigr\} \subset \mathbb{S}^d$ for which
$$ 0 \leq  \frac{1}{n^2} \sum_{i, j=1}^{n} e^{\alpha \left\langle x_i, x_j \right\rangle} -  \frac{1}{|\mathbb{S}^d|^2} \int_{\mathbb{S}^d \times \mathbb{S}^d} e^{\alpha \left\langle x, y \right\rangle} dx dy \leq e^{- c_{\alpha,d} n \log{n}}.$$
\end{theorem}

This type of extremely rapid convergence might be an indicator that
the actual convergence for $\mathcal{L}_{\alpha}$ does indeed happen from below: our use of
Jensen's inequality applied to a concave function (the logarithm) in the
proof is a further indicator.

\subsection{The Frame Potential.} 
 We conclude with a simple observation: for $\alpha \rightarrow 0^+$, the functional $\mathcal{L}_{\alpha}$
 has a Taylor expansion with quite excellent properties.

\begin{prop}\label{prop:frame} For a fixed set of point $\left\{u_1, \dots, u_n \right\} \subset \mathbb{S}^d$, we have, as $\alpha \rightarrow 0$,
\begin{align*} \sum_{i=1}^{n} \log \left( \sum_{j=1}^{n} e^{\alpha \left\langle u_i, u_j \right\rangle}\right) &=  n\log{n} +  \frac{\alpha}{n} \left\| \sum_{i=1}^{n}u_i \right\|^2 + \frac{\alpha^2}{2n}  \sum_{i,j=1}^n \left\langle u_i, u_j\right\rangle^2 \\
&\qquad - \frac{\alpha^2}{2n}  \sum_{i=1}^{n}  \left\langle u_i, \sum_{j=1}^{n} u_j \right\rangle^2 + \mathcal{O}(\alpha^3).
\end{align*}
\end{prop}
This expansion has an interesting property: if $\alpha$ is quite small, then the linear term dominates and minimizers of the energy functional will be forced to have $\| \sum_{i=1}^{n}u_i \|$ quite small or possibly even 0.
This has implications for the third term which will then also be small.
As such we would expect that there is an emerging effective energy given by
$$ E(u_1, \dots, u_n) = \frac{\alpha}{n} \left\| \sum_{i=1}^{n}u_i \right\|^2 + \frac{\alpha^2}{2n}  \sum_{i,j=1}^n \left\langle u_i, u_j\right\rangle^2.$$
This object function, however, is strongly connected to the
\textit{frame potential}
$$ F(u_1, \dots, u_n) = \sum_{i,j=1}^n \left\langle u_i, u_j\right\rangle^2.$$
The frame potential was introduced in the seminal work of Benedetto \& Fickus \cite{benedetto} and has since played an important role in frame theory. What is utterly remarkable is that the Frame Potential has a large number of highly structured minimizers (see, for example, Fig.~1). As shown by Benedetto \& Fickus, for any
$ \left\{u_1, \dots, u_n \right\} \subset \mathbb{S}^{d-1}$
$$ F(u_1, \dots, u_n) = \sum_{i,j=1}^n \left\langle u_i, u_j\right\rangle^2 \geq \frac{n^2}{d}$$
with equality if and only if the set of points form a tight frame, i.e. if
$$ \forall~u \in \mathbb{R}^d: \quad \sum_{i=1}^{n} \left\langle u,u_i \right\rangle^2 = \frac{n}{d}\left\| u \right\|^2.$$ 
\begin{center}
\begin{figure}[h!]
\includegraphics[width=0.3\textwidth]{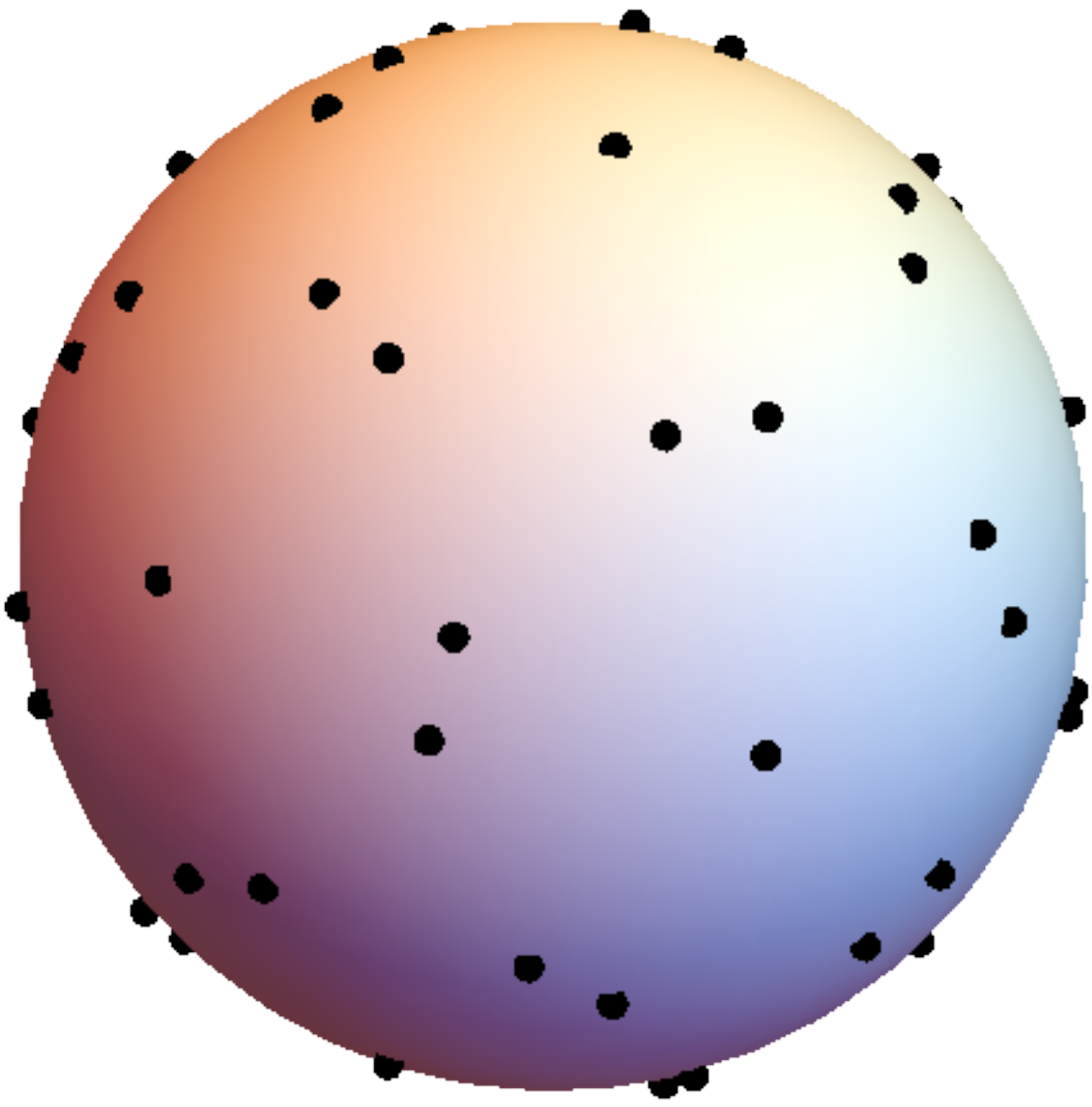}
\caption{The 72 vertices of the Dodecahedron-Icosahedron compound form
  a unit norm tight frame of $\mathbb{R}^3$. The point configuration is also a global
  minimizer of the frame potential.}
\end{figure}
\end{center}
In fact, our effective energy $E(u_1, \dots, u_n)$ may be understood
as the frame potential with an additional strong incentive for the
point configuration to have mean value 0. It would be interesting to
have a better understanding whether $\mathcal{L}_{\alpha}$ inherits
some of the good properties of the Frame Potential for $\alpha$
small, for example, whether it is possible to say anything about minimal energy
configurations of $\mathcal{L}_{\alpha}$ when $n \geq d$ but not going
to $\infty$.

\medskip 

\section{Proof for Theorem~\ref{thm:symvarp}}\label{sec:symmetric}
\begin{proof}
Recall the variational problem under consideration
\begin{equation*}
  \min_u \mc{L}_{\alpha}(u) := \min_u \sum_{i=1}^n  \log \Bigl(1 + \sum_{j=1 \atop j\neq i}^n e^{\alpha (\average{u_j, u_i}-1)} \Bigr).
\end{equation*}
Using Jensen's inequality, we have, for any fixed $1 \leq i \leq n$,
\begin{equation}\label{ineq:Jensen1}
  \begin{aligned}
    \frac{1}{n-1} \sum_{j=1 \atop j\neq i}^n e^{\alpha \average{u_j, u_i}} & \geq \exp \Bigl( \frac{1}{n-1} \sum_{j=1 \atop j\neq i}^n \alpha \average{u_j, u_i}  \Bigr)  \\
    & = \exp \Bigl( \frac{\alpha}{n-1} \bigl(\average{U, u_i} - 1\bigr) \Bigr),
  \end{aligned}
\end{equation}
where we introduce the sum 
\begin{equation*}
  U = \sum_{i=1}^n u_i. 
\end{equation*}
Thus, since the logarithm is monotone, 
\begin{equation}\label{ineq:1}
  \begin{aligned}
    \mc{L}_{\alpha}(u) & = \sum_{i=1}^n \log \Bigl( 1 + e^{-\alpha} \sum_{j=1 \atop j\neq i}^n e^{\alpha \average{u_j, u_i}} \Bigr) \\
    & \geq \sum_{i=1}^n \log \Bigl(1 + (n-1) e^{-\frac{n\alpha}{n-1}}
    e^{\frac{\alpha}{n-1} \average{U, u_i}} \Bigr).
  \end{aligned} 
\end{equation}
Note that for any $a, b > 0$, the function
$t \mapsto \log(1 + a e^{bt})$ is convex, applying Jensen's inequality again, we have 
\begin{equation}\label{ineq:Jensen2}
  \begin{aligned}
    \frac{1}{n}  \mc{L}_{\alpha}(u) &\geq
    \frac{1}{n} \sum_{i=1}^n \log \left[1 + (n-1)
    e^{-\frac{n\alpha}{n-1}} \exp \left(\frac{\alpha}{n-1} \average{U, u_i}\right)
    \right] \\
    & \geq \log \left[ 1 + (n-1)
    e^{-\frac{n\alpha}{n-1}}   \exp\left(\frac{\alpha}{n-1} \frac{1}{n} \sum_{i=1}^{n} \average{U, u_i}\right)
    \right]  \\
    & = \log \left[1 + (n-1)
    e^{-\frac{n\alpha}{n-1}} \exp\left(\frac{\alpha}{n-1} \frac{1}{n} \|U\|^2\right)    \right] \\
    & \geq \log \Bigl(1 + (n-1)
    e^{-\frac{n\alpha}{n-1}} \Bigr),
  \end{aligned}
\end{equation}
Therefore, we arrive at
\begin{equation}
  \begin{aligned}
    \mc{L}_{\alpha}(u) \geq n \log \Bigl(1 + (n-1)
    e^{-\frac{n\alpha}{n-1}} \Bigr). 
  \end{aligned}
\end{equation}
To see when the minimum is achieved, note that in
\eqref{ineq:Jensen1}, due to strict convexity of the exponential function, equality only holds when, for $j \neq i$,
$$ \left\langle u_j, u_i \right\rangle = c_i \qquad \mbox{is independent of}~j.$$
Equality in \eqref{ineq:Jensen2} only holds if
$$  \left\langle U, u_i \right\rangle  = c\qquad \mbox{independent of}~i.$$
Finally, inequality in the last part of \eqref{ineq:Jensen2} only holds when $U = 0$.
If $U=0$, then
$$ c = \left\langle U, u_i \right\rangle =  \left\langle \sum_{j=1}^{n} u_j, u_i \right\rangle = 1 +  \left\langle \sum_{j=1 \atop j \neq i}^{n} u_j, u_i \right\rangle =  1 + (n-1)c_i.$$
This shows that $c_i$ does not actually depend on $i$ and thus 
$\average{u_i, u_j} \equiv c$ for some constant $c$ whenver $i$ is
different from $j$. Thus, we conclude
\begin{equation}
  0 = \left\| \sum_{i=1}^{n} u_i \right\|^2 = n + \sum_{i\neq j} \average{u_i, u_j} = n + n (n-1) c. 
\end{equation}
This implies that 
$$\average{u_i, u_j} = -\frac{1}{n-1}.$$
 Therefore,
the global minimum is achieved if and only if when $\{u_i\}_{i=1}^{n}$ form a simplex equiangular tight frame
(recall $d \geq n-1$ so it is achievable).
\end{proof}

\section{Proof of
  Theorem~\ref{thm:crossentropy}}\label{sec:crossentropy}
\begin{proof}
The proof follows along similar lines as the proof of Theorem~\ref{thm:symvarp}. Recall 
\begin{equation*}
  \mc{L}(u, v) = \sum_{i=1}^n \log \Biggl(
  \frac{\sum_{j=1}^n e^{\average{v_j, u_i} }}{e^{\average{v_i, u_i}}}   \Biggr)= \sum_{i=1}^n \log \Biggl( 1 + 
  \sum_{j=1 \atop j \neq i }^n e^{\average{v_j - v_i, u_i} } 
  \Biggr)
\end{equation*}
Applying Jensen's inequality, we have, for fixed $1 \leq i \leq n$,
\begin{equation}\label{ineq:Jensen3}
  \begin{aligned}
    \sum_{j=1 \atop j\neq i}^n e^{\average{v_j - v_i, u_i}} &=  e^{-\average{v_i, u_i }} \sum_{j=1 \atop j\neq i}^n e^{\average{v_j , u_i}} 
    \\
    &\geq (n-1)  e^{-\average{v_i, u_i }}  \exp \Bigl(\frac{1}{n-1} \sum_{j=1 \atop j\neq i}^n  \average{v_j, u_i}  \Bigr) \\
    & = (n-1) e^{-\average{v_i, u_i }}  \exp  \left(\frac{ \average{V, u_i} - \average{v_i, u_i} }{n-1} \right) \\
    & = (n-1)  \exp  \left(\frac{ \average{V, u_i} - n \average{v_i, u_i} }{n-1} \right), 
  \end{aligned}
\end{equation}
where we denote the sum of $v_i$ as
\begin{equation*}
  V = \sum_{i=1}^n v_i. 
\end{equation*}
Thus, using the monotonicity of logarithm,
\begin{equation}
  \begin{aligned}
    \mc{L}(u, v) & = \sum_{i=1}^n \log \Bigl(1 + \sum_{j=1 \atop j\neq i}^n e^{\average{v_j- v_i, u_i} } \Bigr) \\
    & \geq \sum_{i=1}^n \log \left[ 1 + (n-1) \exp\left( \frac{ \average{V, u_i} }{n-1} 
      -  \frac{n}{n-1} \average{v_i, u_i} \right)  \right]
  \end{aligned}
\end{equation}
Applying Jensen's inequality to the convex function $t \mapsto \log(1 + a e^{bt})$ for $a, b > 0$, we have
\begin{equation}\label{ineq:Jensen4}
  \begin{aligned}
    \mc{L}(u, v) & \geq n \log \left[1 + (n-1) \exp\left( \frac{1}{n} \sum_{i=1}^{n}  \left(\frac{\average{V, u_i}}{n-1} - \frac{n}{n-1} \average{v_i, u_i} \right)  \right) \right] \\
    & = n \log \left[1 + (n-1) \exp\left(  \frac{1}{n} \left( \frac{\average{V, U} }{n-1}   - \frac{n}{n-1} \sum_{i=1}^{n} \average{v_i, u_i} \right)  \right) \right],
  \end{aligned}
\end{equation}
where $U = \sum_i u_i$. 
For the equalities to hold in the above inequalities
\eqref{ineq:Jensen3} and \eqref{ineq:Jensen4}, we require for some
constants $c_i$ and $c$ such that
\begin{align} \label{cond1}
  & \average{v_j, u_i} = c_i, \qquad \forall\, j \neq i
  \end{align}
 and
 \begin{align} \label{cond2}
  & \frac{\average{V, u_i}}{n-1}  - \frac{n}{n-1}\average{v_i, u_i}  = c, \qquad \forall\, i.
\end{align}
Therefore, in order to find a lower bound on $\mathcal{L}$, we have to solve
$$ \frac{\average{V, U} }{n-1}   - \frac{n}{n-1} \sum_{i=1}^{n} \average{v_i, u_i}  \rightarrow \min.$$
If there is a minimizing configuration of this simpler problem that also satisfies \eqref{cond1} and \eqref{cond2}, then all
inequalities are actually equalities.
The above variational problem is equivalent to maximizing
\begin{equation}\label{eq:maxprob}
  n \sum_{i=1}^{n} \average{v_i, u_i} - \left\langle \sum_{i=1}^{n} v_i, \sum_{i=1}^{n} u_i\right\rangle  = \vec{v}^{\top} \bigl((n \mathbb{I}_n - \bd{1}_n\bd{1}_n^{\top}) \otimes \mathbb{I}_d \bigr) \vec{u}, 
\end{equation}
where $\otimes$ denotes the Kronecker product, $\mathbb{I}_d$ ($\mathbb{I}_n$) denotes a  $d \times d$ ($n \times n$) identity matrix,
$\bd{1}_n$ denotes an all-$1$ $n$-vector, $\vec{u}$ denotes a long
$\RR^{nd}$ column vector formed by concatinating $u_i \in \RR^d$ for
$i = 1, \cdots, n$, and similarly for $\vec{v}$. We note that, being the concatenation of unit vectors, $\| \vec{u}\| = \sqrt{n} = \| \vec{v}\|.$\\

The eigenvalues of a Kronecker product $A \otimes B$ are given by $\lambda_i \mu_j$, where
$\lambda_i$ are the eigenvalues of $A$ and $\mu_j$ are the eigenvalues of $B$. The matrix 
$n \mathbb{I}_n - \bd{1}_n\bd{1}_n^{\top}$ is acting like $n \mathbb{Id}_n$ on vectors having
mean value 0 while sending the constant vector to 0. Its spectrum is thus given by $n$ (with multiplicity
$n-1$) and 0. 
It follows that $ \bigl((n \mathbb{I}_n - \bd{1}_n\bd{1}_n^{\top}) \otimes
\mathbb{I}_d \bigr)$ is symmetric, its largest eigenvalue is $n$ and its smallest eigenvalue is 0. Recalling that  $\| \vec{u}\| = \sqrt{n} = \| \vec{v}\|,$ we
have that (without constraints \eqref{cond1} and \eqref{cond2}) 
\begin{equation} \label{universal}
n \sum_{i=1}^{n} \average{v_i, u_i} - \left\langle \sum_{i=1}^{n} v_i, \sum_{i=1}^{n} u_i\right\rangle  \leq n^2.
\end{equation}
However, setting $\vec{u}$ to be the simplex and $\vec{v} = \vec{u}$, we see that $U = 0 = V$ and we have
equality in \eqref{universal} while simultaneously satisfying the constraints \eqref{cond1} and
\eqref{cond2} with
$$c_i = -\frac{1}{n-1} \qquad \mbox{and} \qquad c = - \frac{n}{n-1}.$$
We will now argue that this is the only extremal example. Using the Spectral Theorem, we see that equality in \eqref{universal}
can only occur if $\vec{u}$ is an eigenvector of the matrix corresponding to the eigenvalue $n$. In that case, we have
$$ \left\langle \vec{v}, \bigl((n \mathbb{I}_n - \bd{1}_n\bd{1}_n^{\top}) \otimes \mathbb{I}_d \bigr) \vec{u} \right\rangle = n \left\langle \vec{v}, \vec{u} \right\rangle \leq n \| \vec{v} \| \| \vec{u}\| \leq n^2.$$
We have equality in Cauchy-Schwarz if and only if $\vec{v} = \lambda \vec{u}$ for some $\lambda > 0$. For
$\alpha =1$, this implies that $\vec{v} = \vec{u}$ and we are back in the symmetric case and can argue as in the proof of Theorem 2. Conditions  \eqref{cond1} and \eqref{cond2} simplify to
\begin{align*}
   \average{u_j, u_i} = c_i, ~ \forall\, j \neq i \qquad \mbox{and} \qquad \frac{\average{U, u_i}}{n-1}  - \frac{n}{n-1}\average{u_i, u_i}  = c, ~ \forall\, i.
\end{align*}
Moreover, by summing over the second condition, we see that we want to minimize
$$ c \cdot n = \frac{\average{U, U} }{n-1}   - \frac{n}{n-1} \sum_{i=1}^{n} \average{u_i, u_i}  \rightarrow \min.$$
We are thus interested in minimizing $c$ which is given by
$$ c =  \frac{\average{U, u_i}}{n-1}  - \frac{n}{n-1}\average{u_i, u_i} = c_i - \left\langle u_i, u_i \right\rangle \geq c_i -1.$$
However, for any set of $n$ unit vectors, the largest inner product between any pair of distinct vectors satisfies
$$ 0 \leq \left\| \sum_{i=1}^{n} u_i \right\|^2 = n + n(n-1) \max_{i \neq j} \average{u_i, u_j}.$$
Thus
$$ c \geq \max_{1 \leq i \leq n} c_i -1 \geq - \frac{1}{n-1} - 1 = - \frac{n}{n-1}.$$
We see that equality is achieved for the simplex. Moreover, in the case of equality, we have to require that $\average{u_i, u_j} = -1/(n-1)$ for any pair of distinct vectors and this characterizes the simplex.
 Moreover, if $\|v_i\| \leq \alpha$, then it is easy
to see that the maximum is achieved when $\vec{v} = \alpha \vec{u}$,
we can again conclude using Theorem~\ref{thm:symvarp}.
\end{proof}

\section{Proof of Theorem~\ref{thm:uniform}}\label{sec:uniform}

We first prove Proposition~\ref{prop:uniformsphere};
Theorem~\ref{thm:uniform} then follows from a limiting argument.

\begin{proof}[Proof of Proposition~\ref{prop:uniformsphere}]
Without loss of generality, we assume $\alpha = 1$ in the proof, and note that 
the energy functional can be rewritten as  
\begin{equation*}
  \begin{aligned}
    \mc{L}(\mu) & = \int_{\SS^d} \log \left( \int_{\SS^d} e^{\average
        {x, y} - 1}      \mu(\rd y)\right) \mu(\rd x) \\
    & = \int_{\SS^d} \log \left( \int_{\SS^d} e^{\average{x, y} -
        (\abs{x}^2 + \abs{y}^2)/2}
      \mu(\rd y)\right) \mu(\rd x) \\
    & = \int_{\SS^d} \log \left( \int_{\SS^d} e^{- \abs{x - y}^2/2}
      \mu(\rd y)\right) \mu(\rd x).
   \end{aligned}
\end{equation*} 
Using the Gaussian convolution identity
\begin{equation*}
  e^{-\abs{x - y}^2/2} =  \frac{1}{(8\pi)^{d/2}} \int_{\RR^d} e^{-\abs{x - z}^2 / 4}  e^{-\abs{z - y}^2 / 4} \ud z
\end{equation*}
and the Jensen's inequality, we have 
\begin{equation*}
  \begin{aligned}
    \mc{L}(\mu) & = \int_{\SS^d} \log \left( \int_{\SS^d} \frac{1}{(8\pi)^{d/2}}  \int_{\RR^d} e^{-\abs{x - z}^2 / 4}  e^{-\abs{z - y}^2 / 4} \ud z \mu(\rd y) \right) \mu(\rd x) \\
    & = \int_{\SS^d} \log \left( \int_{\SS^d} \frac{1}{(4\pi)^{d/2}}  \int_{\RR^d} e^{-\abs{x - z}^2 / 4}  e^{-\abs{z - y}^2 / 4} \ud z \mu(\rd y) \right) \mu(\rd x) -  \frac{d}{2} \log 2 \\
    & \geq  \frac{1}{(4\pi)^{d/2}}  \int_{\RR^d} \int_{\SS^d}  e^{-\abs{x - z}^2 / 4}\log \left( \int_{\SS^d}   e^{-\abs{z - y}^2 / 4} \mu(\rd y) \right) \mu(\rd x) \ud z -  \frac{d}{2} \log 2 \\
    & = \frac{1}{(4\pi)^{d/2}} \int_{\RR^d} \left( \int_{\SS^d} e^{-\abs{x - z}^2 / 4}
      \mu(\rd x) \right) \log \left( \int_{\SS^d} e^{-\abs{z - y}^2 /
        4} \mu(\rd y) \right) \ud z -  \frac{d}{2} \log 2.
  \end{aligned}
\end{equation*}
The above calculation shows that minimizing $\mc{L}(\mu)$ is equivalent to minimizing 
\begin{equation*}
  \mc{G}(\mu) := \int_{\RR^d}  f_{\mu}(z) \log f_{\mu}(z)  \ud z, 
\end{equation*}
where we define the short-hand
\begin{equation*}
  f_{\mu}(z) := \int_{\SS^d}  e^{-\abs{x - z}^2 / 4} \mu(\rd x).
\end{equation*}
Let us write the integral of $\mc{G}(\mu)$ in spherical coordinates
and get
\begin{equation*}
  \mc{G}(\mu) = \int_0^{\infty} \left( \frac{1}{\abs{\SS^d}} \int_{\SS^d}  f_{\mu}(r, \theta) \log f_{\mu}(r, \theta)  \ud \theta \right) r^{d-1}  \ud r.
\end{equation*}
\begin{lemma}
  For any $r$, the integral
  \begin{equation*}
    \int_{\SS^d} f_{\mu}(r, \theta) \ud \theta 
  \end{equation*}
  is independent of $\mu \in \mathcal{P}(\SS^d)$.
\end{lemma}
\begin{proof}
  By definition, we have (letting $z = (r, \theta)$ in the spherical coordinates)
  \begin{equation*}
    \begin{aligned}
      \int_{\SS^d} f_{\mu}(r, \theta) \ud \theta & = \int_{\SS^d} \int_{\SS^d}  e^{-\abs{x - z}^2 / 4} \mu(\rd x) \ud \theta \\
      & =  \int_{\SS^d} \int_{\SS^d} e^{-\abs{x - z}^2 / 4} \ud \theta \mu(\rd x)  \\
      & = \int_{\SS^d} m(r) \mu(\rd x) = m(r), 
    \end{aligned}
  \end{equation*}
  where
  \begin{equation}\label{eq:defm}
    m(r) := \int_{\SS^d} e^{-\abs{x - z}^2 / 4} \ud \theta,  
  \end{equation}
  which is independent of $x$ due to the spherical symmetry, and hence
  is only a function of $r$.
\end{proof}

Now for each fixed $r$, we can consider the  variational problem
\begin{equation*}
  \begin{aligned}
    &  \arg \inf_{\mu} \; \mc{G}_r(\mu) :=  \int_{\SS^d}  f_{\mu}(r, \theta) \log f_{\mu}(r, \theta)  \ud \theta\\
    & \text{s.t.} \; \int f_{\mu}(r, \theta) = m(r), 
  \end{aligned} 
\end{equation*}
where $m(r)$ is defined in \eqref{eq:defm}. 
We note that $\mc{G}_r(\mu)$ in terms of $f_{\mu}(r, \cdot)$ is just
the entropy functional, which is strongly convex and is minimized if
and only if $f_{\mu}(r, \cdot)$ is uniform on $\SS^d$, which is
equivalent to the uniformity of $\mu$.
Since $\mc{G}(\mu) = \int_0^{\infty} \mc{G}_r(\mu) r^{d-1} \ud r$ is a
positive linear combination of the energy functional $\mc{G}_r(\mu)$,
we conclude that the global minimum of $\mc{G}(\mu)$ and hence
$\mc{L}(\mu)$ corresponds to uniform probability distribution on $\SS^d$.
\end{proof}

We are now ready to prove Theorem~\ref{thm:uniform}. 
\begin{proof}[Proof of Theorem~\ref{thm:uniform}]
  Let $\mu_n$ be a sequence of probability measures corresponding to
  minimizers for $n = 1, 2, \cdots$. Since $\SS^d$ is compact, it
  suffices to prove that any weakly convergent subsequence $\mu_{n_k}$
  would converge to the uniform measure. Denote $\mu$ the limit, and
  define
  \begin{equation*}
    h_{n_k}(x) := \log \left( \int_{\SS^d} e^{-\abs{x-y}^2/2} \mu_{n_k}(\rd y)  \right), 
  \end{equation*}
  and
  \begin{equation*}
    h(x) := \log \left( \int_{\SS^d} e^{-\abs{x-y}^2/2}  \mu(\rd y)  \right). 
  \end{equation*}
  As $e^{-\abs{x-y}^2/2}$ is a smooth function in $x$ and $y$ and is
  bounded from below by $e^{-2}$ for $x, y \in \SS^d$,
  $\mu_{n_k} \rightharpoonup \mu$ implies that $h_{n_k}$ converges to
  $h(x)$ uniformly on $\SS^d$, and hence
  \begin{equation*}
    \mc{L}(\mu_{n_k}) = \int_{\SS^d} h_{n_k}(x) \mu_{n_k}(\rd x) \to \int_{\SS^d} h(x) \mu(\rd x) = \mc{L}(\mu). 
  \end{equation*}
  Thus, the functional $\mc{L}$ is weakly continuous on
  $\mc{P}(\SS^d)$. Since $\mu_{n_k}$ corresponds to a minimizer of the
  variational problem for $n_k$ points, using the upper bound
  Theorem~\ref{upper}, the limit $\mu$ minimizes $\mc{L}$ on
  $\mc{P}(\SS^d)$, which implies by
  Proposition~\ref{prop:uniformsphere} that $\mu$ is the uniform
  probability measure on $\SS^d$. \end{proof}

 \section{Proof of Theorem \ref{upper}}
 \subsection{Outline.}
 We start with Jensen's inequality: since the logarithm is concave, we obtain
$$  \sum_{i =1}^{n} \log \left( \sum_{j=1}^{n} e^{\alpha \left\langle x_i, x_j \right\rangle} \right) \leq n \log\left(   \frac{1}{n} \sum_{i, j=1}^{n} e^{\alpha \left\langle x_i, x_j \right\rangle} \right).$$
For the rest of the proof, it suffices to understand this double sum. We will prove that there exists a sequence of positive $a_k > 0$ such that
$$  \sum_{i, j=1}^{n} e^{\alpha \left\langle x_i, x_j \right\rangle} = \sum_{k=0}^{\infty} a_k \left| \sum_{\ell = 1}^{n} \phi_k(x_{\ell})\right|^2,$$
where $\phi_k$ denotes the $k-$th spherical harmonic. Recalling that the $0-$th spherical harmonic is a constant normalized in $L^2$, we see that
$$ \phi_0(x) = \frac{1}{\sqrt{|\mathbb{S}^d|}}.$$
From this, we obtain 
 $$   \sum_{i, j=1}^{n} e^{\alpha \left\langle x_i, x_j \right\rangle} =\sum_{k=0}^{\infty} a_k \left| \sum_{\ell = 1}^{n} \phi_k(x_{\ell})\right|^2 \geq a_0 \frac{n^2}{|\mathbb{S}^d|}.$$
 We can moreover determine the constant $a_0$: by plugging in randomly
 chosen points (independently and identically distributed with respect to the uniform measure), we see that
 $$ \lim_{n \rightarrow \infty} \frac{1}{n^2}   \sum_{i, j=1}^{n} e^{\alpha \left\langle x_i, x_j \right\rangle}  = \frac{1}{|\mathbb{S}^d|^2} \int_{\mathbb{S}^d \times \mathbb{S}^d}   e^{\alpha \left\langle x, y \right\rangle} dx dy$$
 while simultaneously
 $$  \lim_{n \rightarrow \infty} \frac{1}{n^2}   \sum_{k=0}^{\infty} a_k \left| \sum_{\ell = 1}^{n} \phi_k(x_{\ell})\right|^2 = \frac{a_0}{|\mathbb{S}^d|}$$
 and therefore
 $$ a_0 =  \frac{1}{|\mathbb{S}^d|} \int_{\mathbb{S}^d \times \mathbb{S}^d}   e^{\alpha \left\langle x, y \right\rangle} dx dy.$$
 If we can prove that $a_k > 0$, then this would imply that
  $$   \sum_{i, j=1}^{n} e^{\alpha \left\langle x_i, x_j \right\rangle} \geq  \frac{n^2}{|\mathbb{S}^d|^2} \int_{\mathbb{S}^d \times \mathbb{S}^d}   e^{\alpha \left\langle x, y \right\rangle} dx dy.$$
 We will show that this is indeed the case and that one can find sets of $n$ points for which the expression is not much larger than that.
 
 \subsection{The Expansion.}
We will now prove the desired expansion. We start by expanding the square
$$  \sum_{k=0}^{\infty} a_k \left| \sum_{\ell = 1}^{n} \phi_k(x_{\ell})\right|^2 = \sum_{k=0}^{\infty} a_k  \sum_{i,j = 1}^{n} \phi_{k}(x_i) \phi_{k}( x_j).$$
At this point we start using a property of the sphere: by grouping spherical harmonics with respect to the Laplacian eigenvalue and prescribing that $a_k$ be constant for all spherical harmonics with the same eigenvalue, we can rewrite the sum as
$$  \sum_{k=0}^{\infty} a_k  \sum_{i,j = 1}^{n} \phi_{\ell}(x_i) \phi_{\ell}( x_j) = \sum_{\ell =0}^{\infty} a_{\ell} \sum_{k \atop -\Delta \phi_k = \ell (\ell+1) \phi_k}\sum_{i,j = 1}^{n} \phi_{k}(x_i) \phi_{k}( x_j)  .$$
However, on $\mathbb{S}^d$ the spherical harmonics are ordered in bands we have the addition formula valid for all $x, y \in \mathbb{S}^d$,
$$ \sum_{k \atop -\Delta \phi_{k} = \ell (\ell+1) \phi_{k}} \phi_{k}(x) \phi_{k}(y) = \frac{\ell+ \lambda}{\lambda} C_{k}^{\lambda} (\left\langle x, y\right\rangle),$$
where $C_{k}^{\lambda}$ are the Gegenbauer polynomials and
$$ \lambda = \frac{d-1}{2}.$$
This leads to
$$  \sum_{k=0}^{\infty} a_k \left| \sum_{\ell = 1}^{n} \phi_k(x_{\ell})\right|^2 = \sum_{\ell=0}^{\infty} a_{\ell} \frac{\ell + \lambda}{\lambda} \sum_{i,j=1}^{n} C_{\ell}^{(d-1)/2}\left( \left\langle x_i, x_j\right\rangle\right),$$
where, by an abuse of notation, we exploit that (by assumption) the $a_k$ coincide whenever the two spherical harmonics share the same Laplacian eigenvalue.
Using \cite[Prop. 2.2]{bilyk}, we see that if all the coefficients in the expansion
$$ e^{\alpha x} = \sum_{k=0}^{\infty} b_k C_k^{(d-1)/2}(x) \qquad \mbox{for}~x\in [-1,1]$$
are positive, then the function is positive on $\mathbb{S}^d$ and then \cite[Lemma 2.3]{bilyk} implies that the expansion converges uniformly. It suffices to show that $b_k > 0$.  The Gegenbauer polynomials $C_k^{\alpha}$ are orthogonal on $[-1,1]$ with respect to the weight 
$$ w(z) = (1-z^2)^{\alpha - \frac{1}{2}}.$$
Making an ansatz
$$ e^{\alpha x} = \sum_{k=0}^{\infty} b_k C_k^{(d-1)/2}(x),$$
we see that the coefficient $b_k$ is given by
$$ b_k = \frac{ \int_{-1}^{1} e^{\alpha x} (1-x^2)^{(d-2)/2} C_k^{(d-1)/2}(x) dx}{  \int_{-1}^{1}  (1-x^2)^{(d-2)/2} \left(C_k^{(d-1)/2}(x)\right)^2 dx}.$$
The denominator has a closed form expression which we will abbreviate by $\alpha_1(k,d)$
$$  \alpha_1(k,d) = \int_{-1}^{1}  (1-x^2)^{(d-2)/2} \left(C_k^{(d-1)/2}(x)\right)^2 dx = \frac{\pi 2^{2- d} \cdot \Gamma(k + d -1)}{k! \cdot (k+(d-1)/2) \cdot \Gamma\left(\frac{d-1}{2}\right)^2}.$$
It remains to understand the numerator.
Gegenbauer polynomials have a Rodrigues formula which is as follows:
$$ C_k^{(d-1)/2} =  \frac{(-1)^k}{2^k \cdot k!}  \frac{\Gamma(d/2) \cdot  \Gamma(k + d-1)}{ \Gamma(d-1) \cdot \Gamma(d/2 + k)} (1-x^2)^{-(d-2)/2} \frac{d^k}{dx^k}\left[ (1-x^2)^{k + (d-2)/2}\right].$$
We will abbreviate the constant by $\alpha_2(k,d)$, i.e. 
$$ C_k^{(d-1)/2}(x) =  (-1)^k  \alpha_2(k,d) (1-x^2)^{-(d-2)/2} \frac{d^k}{dx^k}\left[ (1-x^2)^{k + (d-2)/2}\right],$$
where
$$ \alpha_2(k,d) = \frac{1}{2^k \cdot k!}  \frac{\Gamma(d/2) \cdot  \Gamma(k + d-1)}{ \Gamma(d-1) \cdot \Gamma(d/2 + k)}.$$
Therefore
$$ \alpha_1(k,d) \cdot b_k = \alpha_2(k,d)  \int_{-1}^{1} e^{\alpha x} (-1)^k \frac{d^k}{dx^k} \left[ (1-x^2)^{k+ (d-2)/2 } \right] dx.$$
It is easy to see that $b_0 > 0$. We now distinguish the cases $d=2$ and $d \geq 3$. Let us first assume that $d=2$. 
We see that
$$  \frac{d^{\ell}}{dx^{\ell}} \left[ (1-x^2)^{k } \right] \big|_{x=-1,1} = 0 \qquad \mbox{for} \quad \ell=0,1,\dots,k-1.$$
We can thus use integration by parts and get
$$ \alpha_1(k,d) \cdot b_k = \alpha_2(k,d) \cdot \alpha^k \cdot \int_{-1}^{1} e^{\alpha x}   \left[ (1-x^2)^{k+ (d-2)/2 } \right] dx > 0.$$
The same argument applies to $d \geq 3$. In that case we even have
$$ \frac{d^k}{dx^k} \left[ (1-x^2)^{k+ (d-2)/2 } \right]  \big|_{x=-1,1} = 0$$
and can again integrate by parts $k$ times to obtain
$$ \alpha_1(k,d) \cdot b_k = \alpha_2(k,d) \cdot \alpha^k  \int_{-1}^{1} e^{\alpha x}   (1-x^2)^{k+ (d-2)/2 }  dx > 0.$$
Altogether, we see that $b_k > 0$ for $\alpha > 0$ and thus there exists a sequence of positive $a_k$ such that
$$  \frac{1}{n^2} \sum_{i, j=1}^{n} e^{\alpha \left\langle x_i, x_j \right\rangle} = \sum_{k=0}^{\infty} a_k \left|  \frac{1}{n} \sum_{\ell = 1}^{n} \phi_k(x_{\ell})\right|^2.$$

\subsection{Obtaining Quantitative Estimates.}
We will now go through the argument in the preceding section with the goal of getting quantitative estimates on $b_k$, where
$$ \alpha_1(k,d) \cdot b_k = \alpha_2(k,d) \cdot \alpha^k \cdot \int_{-1}^{1} e^{\alpha x}    (1-x^2)^{k+ (d-2)/2 }  dx > 0$$
and $\alpha_1(k,d)$ and $\alpha_2(k,d)$ are given above in closed form. We are interested in bounds from above, it therefore
suffices to estimate the integral. We see that the integral decays at a polynomial rate -- this is perhaps not all that relevant and
we bound very roughly
$$  \int_{-1}^{1} e^{\alpha x}    (1-x^2)^{k+ (d-2)/2 } \leq 2 e^{\alpha}.$$
Therefore
$$ 0 \leq b_k \leq 2 e^{\alpha} \cdot \alpha^k \cdot  \frac{\alpha_2(k,d)}{\alpha_1(k,d)}$$
We have, ignoring factors that depend solely on the dimension $d$,
\begin{align*}
\frac{\alpha_2(k,d)}{\alpha_1(k,d)} &=  \frac{k! (k+(d-1)/2) \cdot \Gamma\left(\frac{d-1}{2}\right)^2}{\pi 2^{2- d} \cdot \Gamma(k + d -1)}  \frac{1}{2^k \cdot k!}  \frac{\Gamma(d/2) \cdot  \Gamma(k + d-1)}{ \Gamma(d-1) \cdot \Gamma(d/2 + k)}\\
&\lesssim_d   \frac{1}{2^k}  \frac{ (k+(d-1)/2)}{ \Gamma(d/2 + k)}
\end{align*}
which decays faster than any exponential in $k$. We see that this is inherited by the coefficient $b_k$ which satisfies
$$ 0 \leq b_k \lesssim_d e^{\alpha} \cdot  \frac{\alpha^k}{2^k}  \frac{ (k+(d-1)/2)}{ \Gamma(d/2 + k)}.$$
We note that
$$ \lim_{k \rightarrow \infty} \log\left(\frac{1}{b_k} \right)\frac{1}{k \log{k}} =1.$$ 

\subsection{Proof of Theorem \ref{upper}}
The behavior of these coefficients, decaying faster than exponential
in $m$, has a number of interesting consequences. First and foremost,
it means that
$$  \sum_{k=0}^{\infty} a_k \left|  \frac{1}{n} \sum_{\ell = 1}^{n} \phi_k(x_{\ell})\right|^2 \qquad \mbox{being small}$$
is really a statement about the distribution of the measure with respect to the first few spherical harmonics. Having a large error with regards to
some intermediate spherical harmonic is barely detectable -- in particular, deducing structural statements about the points via energy arguments
is presumably more difficult than it is for other kernels with slower decay in the coefficients. Another consequence is that we expect a fairly `flat' energy
landscape.

\begin{proof}[Proof of Theorem \ref{upper}] We use the representiation
$$  \frac{1}{n^2} \sum_{i, j=1}^{n} e^{\alpha \left\langle x_i, x_j \right\rangle} = \sum_{k=0}^{\infty} a_k \left|  \frac{1}{n} \sum_{\ell = 1}^{n} \phi_k(x_{\ell})\right|^2$$
and the fact that the $k=0$ term corresponds to the integral. Thus
$$  \frac{1}{n^2} \sum_{i, j=1}^{n} e^{\alpha \left\langle x_i, x_j \right\rangle} -  \frac{1}{|\mathbb{S}^d|^2} \int_{\mathbb{S}^d \times \mathbb{S}^d} e^{\alpha \left\langle x, y \right\rangle} dx dy  = \sum_{k=1}^{\infty} a_k \left|  \frac{1}{n} \sum_{\ell = 1}^{n} \phi_k(x_{\ell})\right|^2.$$
Now we pick the set $\left\{ x_1, \dots, x_n \right\} \subset \mathbb{S}^d$ to be an optimal spherical design: by a result of Bondarenko, Radchenko \& Viazovska \cite{bondarenko}, there exist $\left\{ x_1, \dots, x_n \right\} \subset \mathbb{S}^d$ such that the average of any polynomial of degree $\mbox{deg}(p) \leq c \cdot n^{1/d}$ evaluated in these points coincides with the global average of the polynomial on the sphere. In particular, the first $\sim_d n$ spherical harmonics are being evaluated exactly. Then, however,
$$ \sum_{k=1}^{\infty} a_k \left|  \frac{1}{n} \sum_{\ell = 1}^{n} \phi_k(x_{\ell})\right|^2 =  \sum_{k=c \cdot n}^{\infty} a_k \left|  \frac{1}{n} \sum_{\ell = 1}^{n} \phi_k(x_{\ell})\right|^2.$$
We now use several rather crude bounds. We note that the $k-$th spherical harmonics has eigenvalue $\sim_d k^{2/d}$ (by Weyl's asymptotic) and use an old result of H\"ormander \cite{hor} to conclude that 
$$ \| \phi_k\|_{L^{\infty}} \lesssim \lambda_k^{\frac{d-1}{4}} \sim k^{\frac{d-1}{2d}}.$$
This implies, for some $c_1 > 0$,
$$ \sum_{k=c \cdot n}^{\infty} a_k \left|  \frac{1}{n} \sum_{\ell = 1}^{n} \phi_k(x_{\ell})\right|^2 \lesssim  \sum_{k = c_1 \cdot n}^{\infty} a_k k^{\frac{d-1}{d}}.$$
By the previous result, we see that the superexponential decay of $a_k$ turns the sum essentially into its largest term and from this the desired bound follows.
\end{proof}

\section{Proof of Proposition~\ref{prop:frame}}
\begin{proof}
We are interested in asymptotics for
$$ \sum_{i=1}^{n} \log \biggl( \sum_{j=1}^{n} e^{\alpha \left\langle u_i, u_j \right\rangle}\biggr) \qquad \mbox{as}~\alpha \rightarrow 0.$$
We have the Taylor expansion
$$ \log{(1+x)} =  x - \frac{x^2}{2} + \mathcal{O}(x^3)$$
and thus, as $\alpha \rightarrow 0$,
\begin{align*}
\sum_{i=1}^{n}\log \biggl( \sum_{j=1}^{n} e^{\alpha \left\langle u_i, u_j \right\rangle}\biggr) &= \sum_{i=1}^{n}\log \biggl( \sum_{j=1}^{n}  1 +   \alpha \left\langle u_i, u_j \right\rangle + \frac{\alpha^2}{2}\left\langle u_i, u_j \right\rangle ^2 + \dots  \biggr) \\
 &=  \sum_{i=1}^{n} \log \biggl( n \biggl( 1 +  \frac{1}{n} \sum_{j=1}^{n}  \alpha \left\langle u_i, u_j \right\rangle + \frac{\alpha^2}{2}\left\langle u_i, u_j \right\rangle ^2 + \dots \biggr) \biggr)\\
  &= n \log{n} + \sum_{i=1}^{n} \log \biggl(  1 +  \frac{1}{n} \sum_{j=1}^{n}  \alpha \left\langle u_i, u_j \right\rangle + \frac{\alpha^2}{2}\left\langle u_i, u_j \right\rangle ^2 + \dots  \biggr).
  \end{align*}
  Using the Taylor expansion of the logarithm and collecting all the terms that are constant, linear or quadratic in $\alpha$, we arrive at
  \begin{align*}
 \sum_{i=1}^{n}\log \biggl( \sum_{j=1}^{n} e^{\alpha \left\langle u_i, u_j \right\rangle}\biggr)  &= n \log{n} +\sum_{i=1}^{n}\biggl( \frac{1}{n} \sum_{j=1}^{n}  \alpha \left\langle u_i, u_j \right\rangle + \frac{\alpha^2}{2}\left\langle u_i, u_j \right\rangle ^2 \biggr) \\
  & \qquad - \frac{1}{2}\sum_{i=1}^{n}\biggl( \frac{1}{n} \sum_{j=1}^{n}  \alpha \left\langle u_i, u_j \right\rangle + \frac{\alpha^2}{2}\left\langle u_i, u_j \right\rangle ^2 \biggr)^2 + \mathcal{O}(\alpha^3).
\end{align*}
The first term simplifies to
$$ \sum_{i=1}^{n}\biggl( \frac{1}{n} \sum_{j=1}^{n}  \alpha \left\langle u_i, u_j \right\rangle + \frac{\alpha^2}{2}\left\langle u_i, u_j \right\rangle ^2 \biggr)= \frac{\alpha}{n}  \left\| \sum_{i=1}^{n}u_i \right\|^2+ \frac{\alpha^2}{2n} \sum_{i,j=1}^n \left\langle u_i, u_j\right\rangle^2.$$
The summand in the second term simplifies to, up to first and second order in $\alpha$
$$  -\frac{1}{2} \biggl( \frac{1}{n}\sum_{j=1}^{n} \alpha \left\langle u_i, u_j \right\rangle + \frac{\alpha^2}{2}\left\langle u_i, u_j \right\rangle ^2 \biggr)^2 = -\frac{\alpha^2}{2n^2} \biggl(\sum_j \left\langle u_i, u_j \right\rangle\biggr)^2 + \mathcal{O}(\alpha^3).$$
\end{proof}

\end{document}